\documentclass{article}

\usepackage[preprint]{neurips_2026}
\usepackage{amsmath}

\usepackage{subfiles}
\usepackage{tikz}
\usetikzlibrary{positioning}
\usepackage{tablefootnote}
\usepackage{tabularx}
\usepackage{threeparttable}
\usepackage{multirow}
\graphicspath{{img/}}
\usepackage{subfiles}
\usepackage{algorithm}
\usepackage{algorithmic}
\usepackage{amssymb}
\usepackage{subcaption}
\usetikzlibrary{calc}

\usepackage[utf8]{inputenc} 
\usepackage[T1]{fontenc}    
\usepackage{hyperref}       
\usepackage{url}            
\usepackage{booktabs}       
\usepackage{amsfonts}       
\usepackage{nicefrac}       
\usepackage{microtype}      
\usepackage{xcolor} 

\usepackage{amsthm}
\theoremstyle{plain}
\newtheorem{theorem}{Theorem}[section]

\theoremstyle{definition}

\theoremstyle{remark}
\newtheorem{remark}[theorem]{Remark}

\definecolor{framebg}{RGB}{198,220,220}
\definecolor{curtain}{RGB}{178,197,195}
\definecolor{skin}{RGB}{205,160,120}
\definecolor{table}{RGB}{175,190,185}
\definecolor{query}{RGB}{36,36,48}
\definecolor{spaceRed}{RGB}{238,86,56}
\definecolor{timeGreen}{RGB}{30,185,40}
\definecolor{localYellow}{RGB}{226,223,38}
\definecolor{globalMagenta}{RGB}{190,55,210}
\definecolor{windowBlue}{RGB}{52,82,190}
\definecolor{heightPink}{RGB}{230,65,180}
\definecolor{textgray}{RGB}{20,24,30}

\newcommand{\baseframe}[2]{%
    \begin{scope}[shift={(#1,#2)}]
        \path[clip] (0,0) rectangle (2.28,2.28);
        \fill[framebg] (0,0) rectangle (2.28,2.28);
        \fill[curtain!75] (1.05,0) rectangle (2.28,2.28);
        \fill[table!65] (0,0) rectangle (2.28,0.8);
        \fill[skin!80] (0.12,1.18) .. controls (0.55,1.15) and (0.8,1.32) .. (1.12,1.36)
            .. controls (1.42,1.4) and (1.52,1.24) .. (1.34,1.12)
            .. controls (0.94,0.98) and (0.5,0.98) .. (0.12,1.02) -- cycle;
        \fill[white!65] (1.42,0.82) rectangle (1.72,1.52);
        \draw[white,line width=3.2pt] (0.57,0) -- (0.57,2.28);
        \draw[white,line width=3.2pt] (1.14,0) -- (1.14,2.28);
        \draw[white,line width=3.2pt] (1.71,0) -- (1.71,2.28);
        \draw[white,line width=3.2pt] (0,0.57) -- (2.28,0.57);
        \draw[white,line width=3.2pt] (0,1.14) -- (2.28,1.14);
        \draw[white,line width=3.2pt] (0,1.71) -- (2.28,1.71);
        \draw[white,line width=1.0pt] (0,0) rectangle (2.28,2.28);
    \end{scope}
}

\newcommand{\querydot}[2]{%
    \begin{scope}[shift={(#1,#2)}]
        \fill[query] (1.42,1.42) circle (0.09);
    \end{scope}
}

\newcommand{\overlayall}[3]{%
    \begin{scope}[shift={(#1,#2)}]
        \fill[#3,opacity=0.52] (0,0) rectangle (2.28,2.28);
    \end{scope}
}

\newcommand{\overlaycell}[4]{%
    \begin{scope}[shift={(#1,#2)}]
        \fill[#4,opacity=0.55] (#3) rectangle ++(0.57,0.57);
    \end{scope}
}

\newcommand{\drawstack}[2]{%
    \baseframe{#1}{5.20}
    \baseframe{#1}{2.60}
    \baseframe{#1}{0.00}
}

\title
{VideoSEMA: a scalable and efficient Mamba-like attention for video understanding} %

\author{%
 Nhat Thanh Tran\thanks{Corresponding author}\\
  Department of Mathematics\\
  University of California, Irvine\\
  Irvine, USA \\
  \texttt{nhattt@uci.edu} \\
  \And
  Fanghui Xue \\
  Qualcomm AI Research\thanks{Qualcomm AI Research is an initiative of Qualcomm Technologies, Inc.} \\
  San Diego, USA \\
  \texttt{fangxue@qti.qualcomm.com} \\
  \And
  Shuai Zhang \\
  Qualcomm AI Research\\ 
  San Diego, USA \\
  \texttt{shuazhan@qti.qualcomm.com} \\
  \And
  Jiancheng Lyu \\
  Qualcomm AI Research\\
  San Diego, USA \\
  \texttt{jianlyu@qti.qualcomm.com} \\
  \And
  Yunling Zheng \\
  Qualcomm AI Research\\
  San Diego, USA \\
  \texttt{yunlzhen@qti.qualcomm.com} \\
  \And
  Yingyong Qi \\
  Qualcomm AI Research\\
  San Diego, USA \\
  \texttt{yingyong@qti.qualcomm.com} \\  
  \And
  Jack Xin \\
  Department of Mathematics\\
  University of California, Irvine\\
  Irvine, USA \\
  \texttt{jack.xin@uci.edu} \\
}

\begin{document}

\maketitle

\begin{abstract}
We present for video understanding (classification) a split space-time attention model, VideoSEMA, consisting of a scalable and efficient Mamba-like attention (SEMA) block in space and a softmax temporal attention in time. In each frame, SEMA attention applies a local window attention in parallel with a global averaging in a Mamba macro-architecture, which is called Mamba-like. Under certain rank conditions, 
we prove that the computationally cheaper split space-time attention is equivalent to full space-time attention. On benchmark K400 data sets, VideoSEMA out-performs heavier 
vision transformer and Mamba models. On benchmark SSv2 data, VideoSEMA leads in top-1 accuracy among models of similar parameter sizes. As image resolution scales up from standard $224^2$ to $1024^2$ on K400 and without fine-tuning, VideoSEMA degrades much more gracefully than VideoMamba in accuracy.
It is promising to extend VideoSEMA to longer videos with a dilated/sparse temporal attention.

\end{abstract}

%

\section{Introduction}
Understanding complex spatiotemporal patterns in videos remains a fundamental challenge in computer vision, with application ranging from classification and segmentation to world modeling for autonomous system such as robotics and self-driving vehicles. With the widespread of usage of multimodal models in combination with language models, recent advances have leveraged transformer based architecture and large foundation models to capture long range dependencies, achieving state of the art performance \cite{gberta_2021_ICML_space_time_attention, Fan_2021_ICCV, Li_2022_CVPR, intern_video2, Wang_2025_ICCV, assran2025vjepa2selfsupervisedvideo}. Despite their success, these models predominately rely on Vision Transformer (ViT) backbones, which incur high computational costs for large video inputs. To alleviate this problem, some works explores Mamba \cite{gu2024mamba} as backbone, which reduces the computation cost for long sequences, such as in VideoMamba \cite{li2024videomambastatespacemodel}. Linear attention is another approach to reduce the cost of classical attention mechanism and WLiT demonstrates that it is an effective way to process video \cite{sun2023wlit}. Lately, Mamba-like attention models have been developed as an efficient replacement of classical softmax attention in image application \cite{mila_2024}. 
Mamba-like means to keep the macro-architecture of Mamba yet embed in it light weight non-Mamba attention blocks. Along this line, SEMA \cite{tran2025semascalableefficientmamba} is designed to mimic the exponential forgetting property of Mamba by an asymptotically guided global approximation of softmax attention. For a duality view and treatment of 
softmax and efficient attentions, 
see \cite{primal_dual_attn_2023,AFIDAF_2024}. An operator splitting and integro-differential equation
perspective of transformer is in \cite{Tai_Trans_2025}. 
In this work, we explore the efficacy of SEMA backbone in video applications.
\medskip

Our main contributions are:
\medskip

\begin{itemize}
   \item We develop a novel video model in the split space-time attention framework based on a Mamba-like scalable and efficient image backbone (SEMA,\cite{tran2025semascalableefficientmamba}). 
   \medskip
   
   \item To our best knowledge, this is the first work to explore Mamba-like backbones for videos. 
   \medskip

   \item We show theoretically that the split space-time attention 
   is equivalent to the full
   space-time attention 
   under specific rank conditions.  
   \medskip
   
   \item  We demonstrate on K400 (SSv2) classification tasks that VideoSEMA is an efficient model, outperforming much larger (comparable) parameter and flops size transformer and Mamba video models to date. Post-training and without fine-tuning on K400, it is much more robust than VideoMamba \cite{li2024videomambastatespacemodel} as video frames scale up in size.
\end{itemize}

    
    

\section{Related Works}

Video representation learning has evolved from CNN to attention based models to hierarchical and efficient sequence architectures such as Mamba \cite{gu2024mamba}. TimeSformer \cite{gberta_2021_ICML_space_time_attention} demonstrated that factorized space time attention can replace 3D convolutions for video recognition, while MViT and MViTv2  \cite{Fan_2021_ICCV, Li_2022_CVPR} introduced multiscale hierarchical Transformer that improves efficiency and scalability across image and video tasks. More recently, state space models have been explored for long range temporal modeling, with VideoMamba achieving competitive performance on video understanding tasks \cite{li2024videomambastatespacemodel}. At scale, foundation models such as InternVideo2 unify self-supervised, constrastive, and generative objectives for multimodal video understanding \cite{intern_video2}. Other efforts focus on efficiency and prediction such as adaptive token strategies improving training and inference flexibility \cite{Wang_2025_ICCV}, while V-JEPA 2 advanced self-supervised world modeling for motion prediction and long horizon reasoning in video \cite{assran2025vjepa2selfsupervisedvideo}. However, many of these rely on ViT as a vision backbone, whereas our work explores a light weight Mamba-like model (SEMA \cite{tran2025semascalableefficientmamba}) as an alternative backbone to capture spatial and temporal features efficiently.



\section{Methodology}

\subsection{Preliminary}

\subsubsection{Attention and SEMA}
Attention is a core component of Transformer which is the main driving force of deep learning in recent years. Formally, given an input $x\in\mathbb{R}^{n\times d}$, then softmax full attention is defined as:
\begin{equation}\label{eq:full_attn}
    A(Q,K,V) = \text{softmax}(QK^T)V,
\end{equation}
where $Q = xW_Q + b_Q, K = xW_K + b_K, V = xW_V + b_V$, for $W_Q, W_K, W_V\in \mathbb{R}^{d\times d}$ and $b_Q, b_K, b_V\in\mathbb{R}^{n\times d}$. We observe that compute the attention matrix ($\text{softmax}(QK^T)$) requires $\mathcal{O}(n^2)$ operations. Also as $n\rightarrow \infty$, the attention matrix disperses, i.e. tending to zero uniformly \cite{tran2025semascalableefficientmamba}. Thus, it is unable to distinguish the variations across the keys. On the other hand, SEMA \cite{tran2025semascalableefficientmamba} is designed to mimic the recurrent relation of 
Mamba \cite{gu2024mamba} in the large token number limit with an averaging operation to approximate the global aspect of the attention matrix. This alleviates the burden in calculating full attention, with an access to global information in an efficient manner. Concretely, 
\begin{equation}\label{eq:sema}
    SEMA(Q,K,V) := A_w(Q,K,V) + \left [\frac{1}{n}\sum_{j=1}^{n} v_j\right ],
\end{equation}
where $[\cdot]$ broadcasts the row $n$ times to permit matrix addition and $A_w$ is window attention \cite{Liu_2021_ICCV_Swin} defined as:
\begin{equation}\label{eq:win_attn_proof}
        A_w(Q, K, V):= 
        \begin{bmatrix}
              \dfrac{\sum_{j\in J(1)}\exp(q_1k_j^T)v_j}{\sum_{i\in J(1)} \exp(q_1k_i^T)}\\
            \vdots \\
             \dfrac{\sum_{j\in J(n)}\exp(q_nk_j^T)v_j}{\sum_{i\in J(n)} \exp(q_nk_i^T)}
        \end{bmatrix},
    \end{equation}
    for some index set $J(m)$. An example is $J(m) = \{Mw +1,\dots, (M+1)w\}$, where $M = \lfloor\frac{m-1}{w}\rfloor$.
\medskip

The latter term of Eq. (\ref{eq:sema}) is proven to be a good approximation for $\text{softmax}(QK^T)$ as $n\rightarrow \infty$ under realistic assumption of the input feature $x$, and also with high probability that such an approximation holds \cite{tran2025semascalableefficientmamba}. Thus SEMA is an effective mechanism to process large input sequence with $\mathcal{O}(n)$ computational complexity. 


\subsubsection{Mamba as a Recursive Attention}
Mamba \cite{gu2024mamba} is a state space models to handle long input sequence with computational complexity of $\mathcal{O}(n)$. For complete derivation of state space models, we refer 
to \cite{gu2024mamba, NEURIPS2020_ssm_hippo}. Concretely, Mamba is a map from $x$ to $y$ through the dynamical system:
\begin{align}\label{eq:mamba}
h_t &= A_t \odot  h_{t-1} + B_t( \Delta_t \odot x_t),\\
y_t &= C_t h_t + D \odot x_t,
\end{align}
where $x_t, \Delta_t \in \mathbb{R}^{1\times d}, A_t, h_t \in \mathbb{R}^{d\times d}, B_t \in\mathbb{R}^{d\times 1}$ and $y_t\in \mathbb{R}^{1\times d}, C_t \in\mathbb{R}^{1\times d}, D\in\mathbb{R}^{1\times d}$, and $\odot$ denotes the Hadamard product.  
It follows in discrete form that 
\begin{equation}\label{eq:mamba_linear_attn_rep}
    y_m = \sum_{i=1}^m q_m \Tilde{k}_i^T\Tilde{v}_i + D\odot x_m,
\end{equation}
where $\Tilde{k}_i^T = \left(\prod_{j=1}^{m-i} A_{m-(j-1)}\right) \odot B_i$, $\Tilde{v}_i = \Delta_i \odot x_i$, $\prod$ is short for multiple matrix products in the elementwise (Hadamard) sense, with the convention that $\prod$ acts as identity if the upper index is zero.
The first term in (\ref{eq:mamba_linear_attn_rep}) 
reveals the $(Q,K,V)$ structure
implicit in Mamba, while the second term
can be understood as a skip connection (a casual masked attention).

\subsection{Proposed Method - VideoSEMA}

Given an input video $x\in\mathbb{R}^{C\times T\times h \times w}$. VideoSEMA first tokenizes using a 3D convolution (e.g. kernel=(1,4,4)) to project the input video into $X\in\mathbb{R}^{C\times T\times H \times W}$ of non-overlapping spatiotemporal patches where $H= h/4, W=w/4$. Second, we append a learnable classifier token at the end of the sequence. The last step of pre-processing is to add learned positional embedding and then temporal embedding. Concretely,
\begin{subequations}
\begin{align}
    X &= \text{3DConv}(x) \\
    X &= [X, X_{cls}] + \text{Emb}_{pos} + \text{Emb}_{temp},
\end{align}
\end{subequations}
where $X_{cls}\in\mathbb{R}^{H\times W\times C}$ is classifier token, $\text{Emb}_{pos}\in\mathbb{R}^{H\times W\times C}$ and $\text{Emb}_{temp}\in\mathbb{R}^{T\times C}$ are learned positional and temporal embedding respectively. Here the addition is broadcast along the appropriate dimension to enable matrix addition.
\medskip

Next, we will process the spatial and temporal information of the input via the VideoSEMA block. The VideoSEMA block is constructed as 
\begin{subequations}    
\begin{align}
    y &= Spatial\_SEMA(X),\label{vsema_spt}\\
    y' &= LayerNorm(y),\\
    z' &= Temp\_Attn(y'),\label{vsema_tmp}\\
    z &= X + z', 
\end{align}
\end{subequations}
where spatial function operates over the $H\times W$ dimension of the input, while the temporal function operates over the $T$ dimension. To be precise, for $Q,K,V\in\mathbb{R}^{T\times H \times W \times C}$, we have
\begin{align}
 Spatial\_SEMA(q_{t,h,w})  &:= \dfrac{1}{HW} \sum_{i=1}^{H} \sum_{j=1}^{W} v_{t,i,j}\notag \\ & + \sum_{l,p\in I(h,w)} \dfrac{\exp(q_{t,h,w}k_{t,l,p}^T)}{\sum_{i,j\in I(h,w)}\exp(q_{t,h,w}k_{t,i,j}^T)} v_{t,l,p},
\end{align} 
where $I(h,w)$ is the index set for the window, e.g. window size of $7$. And temporal attention
\begin{equation}
Temp\_Attn(q_{t,h,w}) :=
\sum_{i=1}^{T} \dfrac{\exp(q_{t,h,w}k_{i,h,w}^T)}{\sum_{j=1}^{T}\exp(q_{t,h,w}k_{j,h,w}^T)} v_{i,h,w}. 
\end{equation}


The alternating spatial and temporal attention 
treatment in (\ref{vsema_spt}) and (\ref{vsema_tmp}) can be viewed as an efficient operator splitting approximation of the full SEMA 
obtained by directly extending SEMA formula (\ref{eq:sema})
to space and time. For video frames of moderate lengths considered here, a softmax attention in (\ref{vsema_tmp}) is affordable and effective as supported by our ablation study, making linear complexity approximations unnecessary in time. As seen later, VideoSEMA outperforms much heavier 
networks (Tab.1). 
As a future direction for handling longer videos, a sparse or dilated attention \cite{Longnet_2023} in the temporal attention step is a promising alternative  to leverage continuity of information in time at reduced computational costs. 
\medskip

VideoSEMA's macro-structure is shown in Fig.\ref{fig:VideoSEMA_macro}, and is 
repeated in the network. 
Between two adjacent repeats, we use a convolutional layer to down-sample the spatial dimension of the video signal to produce a hierarchical network 
(similar to  \cite{mila_2024,tran2025semascalableefficientmamba,Liu_2021_ICCV_Swin}). Lastly in Algorithm \ref{alg:VideoSEMA}, the representation of the $[CLS]$ token is normalized before passing through a linear classification head.

\begin{figure}
    \centering
    \includegraphics[width=0.8\linewidth]{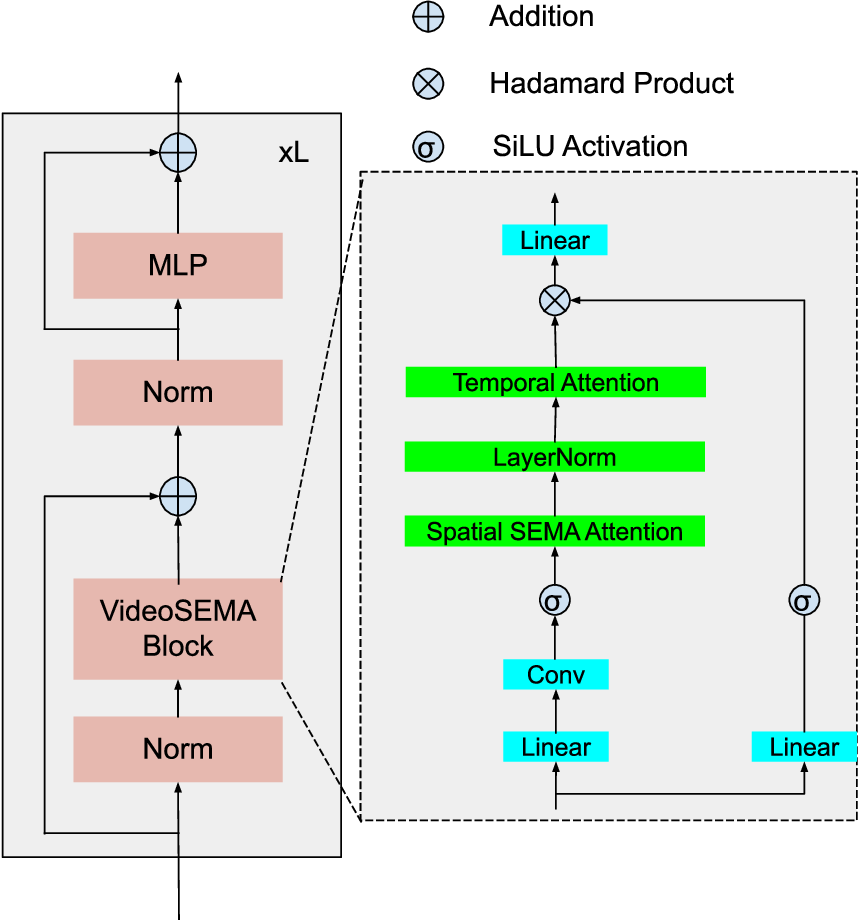}
    \caption{Overview of VideoSEMA macro-structure.}
    \label{fig:VideoSEMA_macro}
\end{figure}


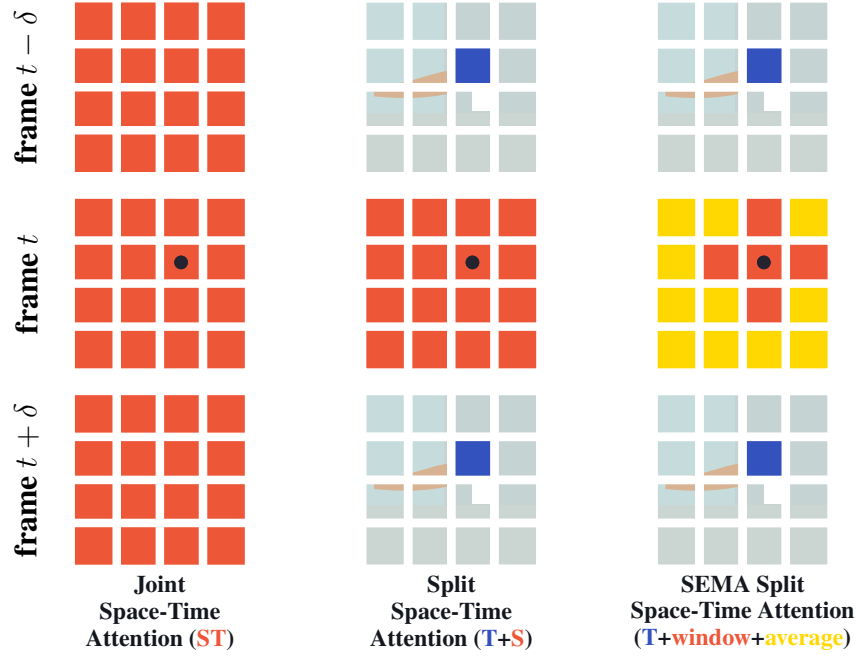
\begin{figure}
    \centering

\begin{tikzpicture}[font=\sffamily]
\fill[white] (-1.20,-1.30) rectangle (10.75,7.80);

\def\xJoint{0.0}
\def\xDivided{3.85}
\def\xSEMA{7.70}

\drawstack{\xJoint}{0}
\drawstack{\xDivided}{0}
\drawstack{\xSEMA}{0}

\overlayall{\xJoint}{5.20}{spaceRed}
\overlayall{\xJoint}{2.60}{spaceRed}
\overlayall{\xJoint}{0.00}{spaceRed}
\querydot{\xJoint}{2.60}

\overlaycell{\xDivided}{5.20}{1.14,1.14}{windowBlue}
\overlayall{\xDivided}{2.60}{spaceRed}
\overlaycell{\xDivided}{0.00}{1.14,1.14}{windowBlue}
\querydot{\xDivided}{2.60}

\overlaycell{\xSEMA}{5.20}{1.14,1.14}{windowBlue}
\overlayall{\xSEMA}{2.60}{yellow!75!orange}
\overlaycell{\xSEMA}{2.60}{1.14,1.14}{spaceRed}
\overlaycell{\xSEMA}{2.60}{0.57,1.14}{spaceRed}
\overlaycell{\xSEMA}{2.60}{1.71,1.14}{spaceRed}
\overlaycell{\xSEMA}{2.60}{1.14,0.57}{spaceRed}
\overlaycell{\xSEMA}{2.60}{1.14,1.71}{spaceRed}
\overlaycell{\xSEMA}{5.20}{1.14,1.14}{windowBlue}
\overlaycell{\xSEMA}{0.00}{1.14,1.14}{windowBlue}
\querydot{\xSEMA}{2.60}

\foreach \x in {\xJoint,\xDivided,\xSEMA} {
    \foreach \y in {5.20,2.60,0.00} {
        \begin{scope}[shift={(\x,\y)}]
            \draw[white,line width=3.2pt] (0.57,0) -- (0.57,2.28);
            \draw[white,line width=3.2pt] (1.14,0) -- (1.14,2.28);
            \draw[white,line width=3.2pt] (1.71,0) -- (1.71,2.28);
            \draw[white,line width=3.2pt] (0,0.57) -- (2.28,0.57);
            \draw[white,line width=3.2pt] (0,1.14) -- (2.28,1.14);
            \draw[white,line width=3.2pt] (0,1.71) -- (2.28,1.71);
            \draw[white,line width=1pt] (0,0) rectangle (2.28,2.28);
        \end{scope}
    }
}

\querydot{\xJoint}{2.60}
\querydot{\xDivided}{2.60}
\querydot{\xSEMA}{2.60}

\node[rotate=90,anchor=center,font=\bfseries\large] at (-0.62,6.34) {frame $t-\delta$};
\node[rotate=90,anchor=center,font=\bfseries\large] at (-0.62,3.74) {frame $t$};
\node[rotate=90,anchor=center,font=\bfseries\large] at (-0.62,1.14) {frame $t+\delta$};

\node[align=center,font=\bfseries\small,text=textgray] at (\xJoint+1.14,-0.62) {Joint\\Space-Time\\Attention (\textcolor{spaceRed}{ST})};
\node[align=center,font=\bfseries\small,text=textgray] at (\xDivided+1.14,-0.62) {Split\\Space-Time\\Attention (\textcolor{windowBlue}{T}+\textcolor{spaceRed}{S})};
\node[align=center,font=\bfseries\small,text=textgray] at (\xSEMA+1.14,-0.62) {SEMA Split\\Space-Time Attention \\(\textcolor{windowBlue}{T}+\textcolor{spaceRed}{window}+\textcolor{yellow!75!orange}{average})};

\end{tikzpicture}
    \caption{Visualization of space time attention types. For illustration, the dark dot denotes the query patch and colored patches show
its self-attention space-time neighborhood under each scheme. Patches without
color are not used for the self-attention computation of the query patch.
Multiple colors within a scheme denote attentions separately applied along
different dimensions, e.g., space and time for \((T+S)\). Note that
self-attention is computed for every single patch in the video clip, i.e., every
patch serves as a query. Although the attention pattern is shown for only two
adjacent frames, it extends in the same fashion to all frames of the clip. }
    \label{fig:spacetime_attention_type}
\end{figure}

\begin{algorithm}[tb]
\caption{Proposed VideoSEMA algorithm. Here cls, temp, pos denote classifier tokens, temporal embedding, and positional embedding respectively. When there are dimensions mismatch in binary operations, there are implicit reshaping/broadcasting to match the dimensions. On the right hand side, we denote the current running shape of $x$.}
\label{alg:VideoSEMA}
    \textbf{Input}: Input $x$\hfill ($C\times T\times H\times W$)\\
    \textbf{Learnable Params}: cls $\in \mathbb{R}^{H\times W\times C}$, temp $\in\mathbb{R}^{T\times C}$, pos $\in\mathbb{R}^{HW\times C}$ 
    \begin{algorithmic}[1]
        \STATE $x =$ reshape($x$) \hfill ($T \times H \times W \times C$)
        \STATE $x =$ concat(($x$, cls), dim=0) \hfill ($(T+1)\times H\times W\times C$)
        \STATE $x$ = reshape($x$) \hfill ($(T+1)\times HW\times C$)
        \STATE $x$ = $x + $pos \hfill ($(T+1)\times HW \times C$)
        \STATE $x = x + $temp \hfill ($(T+1)\times HW \times C$)
        \STATE $x =$ reshape($x$) \hfill ($(T+1)\times HW \times C$)
        \STATE $x$ = VideoSEMA($x$)\hfill ($(T+1)\times HW\times C$)
            \begin{ALC@g}
                \STATE Process spatial information via SEMA
                \STATE Process temporal information via classical attention
            \end{ALC@g}
        \STATE $x$ = AvgPooling($x$)\hfill ($(T+1)\times  C$)
        \STATE \textbf{Return} $x[-1,:]$ \hfill ($1 \times C$)
    \end{algorithmic}

\end{algorithm}

\subsection{Complexity}
Given an input video 
$x\in\mathbb{R}^{T\times H\times W\times C}$, the joint space-time attention treats the video as a sequence of length $N=THW$. Therefore, the attention matrix has size $N\times N$, yielding computational complexity $\mathcal{O}((THW)^2C)$. In practice, this formulation is prohibitively expensive for videos due to the quadratic memory and compute cost of the attention matrix. Consequently, fully joint space time attention is rarely used in large scale video models \cite{gberta_2021_ICML_space_time_attention}.
\medskip

The split space time attention approach works with a factorized (or separate spatial and temporal) attention, thereby lower the computational costs similar to operator splitting \cite{OperSplit_1968,Oper_Split_2017}. In particular, TimeSformer \cite{gberta_2021_ICML_space_time_attention} applies a split 2-step procedure: (1) spatial attention independently in each frame, and (2) temporal attention independently across frames at each spatial location. The  spatial attention complexity in (1) is $\mathcal{O}(T(HW)^2 C)$ as each of the $T$ frames performs full attention over the $HW$ spatial tokens. The temporal attention complexity in (2) is $\mathcal{O}(T^2HWC)$ as each spatial location performs full attention over $T$ tokens. So the total complexity of the split attention design becomes $\mathcal{O}((T^2HW + T(HW)^2 )C)$.
\medskip

In contrast, VideoSEMA replaces full spatial attention with SEMA, whose complexity scales linearly with spatial dimension $\mathcal{O}(THWwC)$ where $w$ denotes the window size. The temporal attention is still performed globally, thus the total computational complexity of VideoSEMA is $\mathcal{O}((T^2HW + THWw )C)$, linear in spatial resolution $HW$. In the datasets of this paper, the number of video frames is at most 32 (i.e. $T\leq 32$ in Tab. 1, and $T\leq 16$ in Tab. 2 and Tab. 3). The main contribution to the complexity comes from spatial resolution 
$HW = (224)^2$. With a linear complexity in $HW$, VideoSema made  considerable computational savings from TimeSformer 
\cite{gberta_2021_ICML_space_time_attention}, as seen in Table \ref{tab:K400_SOTA} on $16\times 224^2$ resolution of K400 dataset where VideoSema's parameter size is a fraction 31/121 of TimeSformer's while achieving better accuracies. We present visualization of each space-time attention in Fig. \ref{fig:spacetime_attention_type}.

\section{Theoretical Explanation}

\subsection{A Simplified Setting}

In this section, we show that in an ideal setting, the split space-time attention can be equivalent to full attention. Given an input $x \in\mathbb{R}^{N\times T\times d}$, where $N, T, d$ represent spatial, temporal and channel dimensions respectively. For simplicity, let $V = x$, then the full attention for a fixed query $q$ becomes: 
\begin{equation}
    A_{ST} = \sum_{i=1}^{N}\sum_{j=1}^{T} \eta_{ij} x_{ij},\;\; \eta_{ij} = \dfrac{\phi(q \, k_{ij}^T)}{\sum_{ij} \phi(q \, k_{ij}^T)}, \label{full-st-attn}
\end{equation}
where 
$k_{ij}$ and $x_{ij}$ are vectors in $\mathbb{R}^{d}$ at space-time location $(i,j)$,  $1\leq i \leq N$, $1\leq j\leq T$.  Here we define $\phi: \mathbb{R} \rightarrow \mathbb{R}^{+}$ to be any continuous function. 

Now we compute split space-time attention $A_{T+S}$ for the fixed query $q$. Without loss of generality, we have the split space-time function as the composite function of space and then time. The order does not matter, as it is up to a swap in the first and second dimensions of the input $x$. We have

\begin{equation}
    A_{T+S} = \sum_{j=1}^{T} \alpha_{j} \left(\sum_{i=1}^{N} \beta_{N(j-1)+i} \, x_{ij}\right), \label{split_attn}
\end{equation}
where $\alpha_j$'s represent attention score in the time dimension, and $\beta_{\cdot}$'s represent the frame-wise attention score. The equivalence $A_{ST} = A_{T+S}$  is realized if the following system of equations hold:

\begin{align}
    &\eta_{ij} = \alpha_j \beta_{N(j-1) + i}, \label{match_full_w}\\
    &\sum_{i,j = 1}^{N,T}\eta_{ij}  = 1,\label{eq:eta_cond}\\
    &\sum_{i=1}^{N} \beta_{N(j-1) + i} = 1,\;\; \forall j \in [1,\dots, T],\label{eq:beta_cond}\\
    &\sum_{j=1}^{T} \alpha_j = 1 \label{eq:alpha_cond}.
\end{align}

An ideal solution for this system exists as follows.
Suppose the standard full attention weights $\eta_{ij}$ are given as in (\ref{full-st-attn}) and so (\ref{eq:eta_cond}) is satisfied. To match these attention scores by the split space-time attention 
(\ref{split_attn}), we let 
\begin{equation}
    \beta_{N(j-1)+i} = \dfrac{\eta_{ij}}{\sum_{l=1}^{N} \eta_{lj}}, \label{fac_beta}
\end{equation}
and 
\begin{equation}
    \alpha_{j} = \sum_{l=1}^{N} \eta_{lj}, \label{fac_alpha}
\end{equation}
then (\ref{eq:beta_cond}) and (\ref{eq:alpha_cond}) hold automatically. 
In practice however, one computes the weights $(\alpha_{\cdot},\beta_{\cdot})$ of the split attention
(\ref{split_attn}) without knowledge of full attention weights. The corresponding factorization (\ref{match_full_w}) 
may not satisfy normalization 
condition (\ref{eq:eta_cond}), and may not be equal to the $(\eta_{ij})$  
in (\ref{full-st-attn}). 
The ideal solution 
(\ref{fac_beta})-(\ref{fac_alpha})
would only be an indirect theoretical target for the training of $A_{T+S}$.

\subsection{Attention Factorization and Rank Conditions}
Now we turn to standard attention. The previous section only addresses a single query, however in practice one works with
multiple queries. Assume we have $M$ queries, and let $m\in \{1,\dots, M\}$ index the query row, and  let $\eta_{mij}$ be the full softmax attention coefficient from query $m$ to spatial location $i$ in frame $j$. The single-query construction above applies row-wise, so define
\begin{equation}
    \alpha_{mj} = \sum_{l=1}^{N}\eta_{mlj},
    \qquad
    \beta_{mi|j} = \dfrac{\eta_{mij}}{\alpha_{mj}}.
\end{equation}

Since softmax coefficients are strictly positive, $\alpha_{mj}>0$ and the definition is well-defined. Moreover,
\begin{equation}
    \sum_{j=1}^{T}\alpha_{mj}=1,
    \qquad
    \sum_{i=1}^{N}\beta_{mi|j}=1,
    \qquad
    \alpha_{mj}\beta_{mi|j}=\eta_{mij}.
\end{equation}
Thus the split output equals the full output for every query:
\begin{equation}
    A_{ST}^{(m)}
    =
    \sum_{j=1}^{T}\sum_{i=1}^{N}\eta_{mij}x_{ij}
    =
    \sum_{j=1}^{T}\alpha_{mj}
    \left(\sum_{i=1}^{N}\beta_{mi|j}x_{ij}\right)
    =
    A_{T+S}^{(m)}.
\end{equation}
Therefore split space time attention is an exact marginal-conditional decomposition of each row of the full attention matrix. The only remaining question is whether a chosen dot-product parameterization can realize the required temporal and spatial logits.
\medskip

For positive normalized attention, $\alpha_{mj}>0$, and these coefficients again satisfy $\alpha_{mj}\beta_{mi|j}=\eta_{mij}$. A particular split attention module realizes this decomposition exactly when its temporal branch can produce the distribution $\alpha_{m,:}$ and its spatial branch can produce the conditional distributions $\beta_{m:|j}$ under the same normalized $\phi$ operation. If $\phi$ is invertible on its positive range, one possible choice of temporal and spatial scores is any $r_{mj}$ and $s_{mij}$ satisfying
\begin{equation}
    \phi(r_{mj}) = c_m\alpha_{mj},
    \qquad
    \phi(s_{mij}) = c_{mj}\beta_{mi|j},
\end{equation}
for arbitrary positive constants $c_m$ and $c_{mj}$.
\begin{theorem}[Exact rank condition for normalized $\phi$]\label{thm:exact_rank_phi_split}
    Assume $\phi:\mathbb{R}\rightarrow\mathbb{R}^{+}$ is invertible on its positive range. Choose positive constants $c_m$ and $c_{mj}$ so that $c_m\alpha_{mj}$ and $c_{mj}\beta_{mi|j}$ lie in the range of $\phi$, and define the target temporal score matrix $R^{\phi}\in\mathbb{R}^{M\times T}$ and target spatial score matrix $S^{\phi}\in\mathbb{R}^{M\times NT}$ by
    \begin{equation}
        R^{\phi}_{mj}
        =
        \phi^{-1}(c_m\alpha_{mj}),
        \qquad
        S^{\phi}_{m,(j,i)}
        =
        \phi^{-1}(c_{mj}\beta_{mi|j}).
    \end{equation}
    Let $d_T,d_S\in\mathbb{N}$ be the temporal and spatial head dimensions. If
    \begin{equation}
        d_T \ge \operatorname{rank}(R^{\phi}),
        \qquad
        d_S \ge \operatorname{rank}(S^{\phi}),
    \end{equation}
    then there exist learned query and key matrices $Q^{\tau} \in\mathbb{R}^{M\times d_T},K^{\tau}\in\mathbb{R}^{T\times d_T}$ and $Q^{S}\in\mathbb{R}^{M\times d_S},K^{S}\in\mathbb{R}^{NT\times d_S}$ such that
    \begin{equation}
        Q^{\tau}(K^{\tau})^{T}=R^{\phi},
        \qquad
        Q^{S}(K^{S})^{T}=S^{\phi}.
    \end{equation}
    Consequently, normalized $\phi$ attention over the temporal scores produces $\alpha_{m,:}$, normalized $\phi$ attention over the spatial scores inside each frame produces $\beta_{m:|j}$, and split space time attention exactly equals full attention for every query row $m$.
\end{theorem}

\begin{proof}
    The rank assumptions imply matrix factorizations $R^{\phi}=Q^{\tau}(K^{\tau})^{T}$ and $S^{\phi}=Q^{S}(K^{S})^{T}$, for example by the compact singular value decomposition. For each query row $m$, normalizing $\phi(R^{\phi}_{m,:})$ over the temporal index gives
    \begin{equation}
        \dfrac{\phi(R^{\phi}_{mj})}{\sum_{n=1}^{T}\phi(R^{\phi}_{mn})}
        =
        \dfrac{c_m\alpha_{mj}}{\sum_{n=1}^{T}c_m\alpha_{mn}}
        =
        \alpha_{mj}.
    \end{equation}
    Likewise, for each fixed $(m,j)$, normalizing $\phi(S^{\phi}_{m,(j,:)})$ over the spatial index gives
    \begin{equation}
        \dfrac{\phi(S^{\phi}_{m,(j,i)})}{\sum_{l=1}^{N}\phi(S^{\phi}_{m,(j,l)})}
        =
        \dfrac{c_{mj}\beta_{mi|j}}{\sum_{l=1}^{N}c_{mj}\beta_{ml|j}}
        =
        \beta_{mi|j}.
    \end{equation}
    \medskip
    
    Therefore the split coefficient is $\alpha_{mj}\beta_{mi|j}=\eta_{mij}$ for all $m,i,j$.
\end{proof}

\begin{remark}[Softmax as a special case]
    If $\phi(x)=\exp(x)$, then the above construction recovers the usual softmax factorization. In this case
    \begin{equation}
        \beta_{mi|j}
        =
        \dfrac{\exp(e_{mij})}{\sum_{l=1}^{N}\exp(e_{mlj})},
        \qquad
        \alpha_{mj}
        =
        \dfrac{\exp(\tau_{mj})}{\sum_{n=1}^{T}\exp(\tau_{mn})},
    \end{equation}
    where the temporal score can be chosen as the frame-level log-sum-exp
    \begin{equation}
        \tau_{mj}
        =
        \log\left(\sum_{l=1}^{N}\exp(e_{mlj})\right),
        e_{mij} = q_m k_{ij}^{T}.
    \end{equation}
    More generally, if the desired normalized coefficients $\eta_{mij}$ are already known, softmax scores can be chosen as
    \begin{equation}
        r_{mj}
        =
        \log(\alpha_{mj}) + c_m,
        \qquad
        s_{mij}
        =
        \log(\beta_{mi|j}) + c_{mj},
    \end{equation}
    since adding a constant to all scores in a normalized softmax does not change the resulting distribution. This is the special case of Theorem \ref{thm:exact_rank_phi_split} with $\phi^{-1}(x)=\log(x)$.
\end{remark}

Now suppose that we are given a full rank condition as in Theorem \ref{thm:exact_rank_phi_split}, then we can realize the matrices $Q^{\tau}, K^{\tau}, Q^S, K^S$ under a typical condition, for example, given $y \in\mathbb{R}^{M\times d_T}$, we want $Q^{\tau} = yW_Q$, for some $W_Q \in\mathbb{R}^{d_T\times d_T}$. Thus we solve the optimization problem
\begin{equation}
    \min_W \|Q^{\tau} - yW\|_2,
\end{equation}
and this has an exact solution if $col(Q^{\tau}) \subset col(y)$.

\begin{theorem}[Low-rank approximation for normalized $\phi$]\label{thm:approx_rank_phi_split}
    Assume $\phi$ is invertible on its positive range, and let $R^{\phi}$ and $S^{\phi}$ be the target score matrices from Theorem \ref{thm:exact_rank_phi_split}. Let $\widehat{R}$ and $\widehat{S}$ be rank-constrained approximations with
    \begin{equation}
        \operatorname{rank}(\widehat{R})\le d_T,
        \qquad
        \operatorname{rank}(\widehat{S})\le d_S.
    \end{equation}
    Define the approximate normalized coefficients
    \begin{equation}
        \widehat{\alpha}_{mj}
        =
        \dfrac{\phi(\widehat{R}_{mj})}{\sum_{n=1}^{T}\phi(\widehat{R}_{mn})},
        \qquad
        \widehat{\beta}_{mi|j}
        =
        \dfrac{\phi(\widehat{S}_{m,(j,i)})}
        {\sum_{l=1}^{N}\phi(\widehat{S}_{m,(j,l)})},
    \end{equation}
    and let $\widehat{\eta}_{mij}=\widehat{\alpha}_{mj}\widehat{\beta}_{mi|j}$. If
    \begin{equation}
        \epsilon_T
        =
        \max_m
        \|\widehat{\alpha}_{m,:}-\alpha_{m,:}\|_1,
        \qquad
        \epsilon_S
        =
        \max_{m,j}
        \|\widehat{\beta}_{m:|j}-\beta_{m:|j}\|_1,
    \end{equation}
    then for every query row $m$,
    \begin{equation}
        \sum_{j=1}^{T}\sum_{i=1}^{N}
        |\widehat{\eta}_{mij}-\eta_{mij}|
        \le
        \epsilon_T+\epsilon_S.
    \end{equation}
    If additionally $\|x_{ij}\|_2\le B$ for all $i,j$, then
    \begin{equation}
        \|\widehat{A}_{T+S}^{(m)}-A_{ST}^{(m)}\|_2
        \le
        B(\epsilon_T+\epsilon_S).
    \end{equation}
\end{theorem}

\begin{proof}
    Add and subtract $\widehat{\alpha}_{mj}\beta_{mi|j}$:
    \begin{align}
        |\widehat{\eta}_{mij}-\eta_{mij}|
        &=
        |\widehat{\alpha}_{mj}\widehat{\beta}_{mi|j}
        -
        \alpha_{mj}\beta_{mi|j}|\\
        &\le
        \widehat{\alpha}_{mj}
        |\widehat{\beta}_{mi|j}-\beta_{mi|j}|
        +
        |\widehat{\alpha}_{mj}-\alpha_{mj}|\beta_{mi|j}.
    \end{align}
    Summing over $i,j$, and using that $\widehat{\alpha}$ and $\beta$ are probability distributions, gives
    \begin{equation}
        \sum_{j=1}^{T}\sum_{i=1}^{N}
        |\widehat{\eta}_{mij}-\eta_{mij}|
        \le
        \|\widehat{\alpha}_{m,:}-\alpha_{m,:}\|_1
        +
        \sum_{j=1}^{T}\widehat{\alpha}_{mj}
        \|\widehat{\beta}_{m:|j}-\beta_{m:|j}\|_1
        \le
        \epsilon_T+\epsilon_S.
    \end{equation}
    The output bound follows from
    \begin{equation}
        \|\widehat{A}_{T+S}^{(m)}-A_{ST}^{(m)}\|_2
        \le
        \sum_{j=1}^{T}\sum_{i=1}^{N}
        |\widehat{\eta}_{mij}-\eta_{mij}|\|x_{ij}\|_2.
    \end{equation}
\end{proof}

\begin{remark}
In particular, one may choose $\widehat{R}$ and $\widehat{S}$ as best rank-$d_T$ and rank-$d_S$ approximations of $R^{\phi}$ and $S^{\phi}$ in Frobenius norm. By the Eckart-Young theorem,
\begin{equation}
    \|R^{\phi}-\widehat{R}\|_{F}^{2}
    =
    \sum_{r>d_T}\sigma_r(R^{\phi})^2,
    \qquad
    \|S^{\phi}-\widehat{S}\|_{F}^{2}
    =
    \sum_{r>d_S}\sigma_r(S^{\phi})^2,
\end{equation}
where $\sigma_r$ is the $r$ singular value. Thus the approximation quality is controlled by the singular value tails of the target temporal and spatial score matrices, together with how sensitively the normalized $\phi$ map converts score errors into coefficient errors.
\end{remark}

\begin{remark}
    We are mostly dealing with long sequence, thus the number of query $M$ or spatial $N$ is large. Therefore the rank condition is quite strict, that is we are most likely to be in the scenario of Theorem \ref{thm:approx_rank_phi_split}. However when we work with short video clips as in SSv2 dataset, i.e. $T$ is small, $Q^{\tau}$ and $K^{\tau}$ may be found exactly during training.
\end{remark}

\section{Experiments}

\begin{table}[ht!]
    \caption{Performance reported on  standard K400 dataset. Here $F\times m\times n$ denotes by $F$ the total FLOPs, $m$ the number of testing segments, $n$ the number of testing crops. Res: resolution, T: transformer. P(M): parameter size in millions. T1: top-1, T5: top-5. C: CNN, CT: CNN +T, M: Mamba, Ml: Mamba-like.}
    \centering
    \begin{tabular}{|l|c|c c| c |c c|}
    \hline
         Method & Type  & P(M) &  FLOPs (G) & Res& T1(\%) & T5(\%)\\
         \hline
         X3D-XL \cite{Feichtenhofer_2020_CVPR} & C & 20   &$194\times 3\times 10$ & $16\times 224^2$&80.4 & 94.6\\
         Swin-T \cite{Liu_2022_CVPR} & T & 28 & $88\times 3\times 4$ & $32\times 224^2$&78.8 & 93.6 \\
         MViTv1-B \cite{Fan_2021_ICCV} & CT & 37 & $70\times 1\times 5$ & $32\times 224^2$ & 80.2 & 94.4 \\
         MViTv2-S \cite{Li_2022_CVPR} & CT & 35 & $64\times 1\times 5$ & $16\times 224^2$ & 81.0 & 94.6 \\
         Uniformer-S \cite{li2022uniformer} & CT & 21 & $42\times 1\times 4$ & $16\times 224^2$ & 80.8 & 94.7 \\
         WLiT \cite{sun2023wlit} & T & 22 & $21\times 1 \times 4$ & $8\times 224^2$ & 74.6 & 92.0 \\
         \hline
         TimeSformer-L \cite{gberta_2021_ICML_space_time_attention} & T & 121 & $2380\times 3\times 1$ & $16\times 224^2$ & 80.7 & 94.7 \\
         Uniformer-S \cite{li2022uniformer} & CT & 311 & $3992\times 3\times 4$ & $16\times 224^2$ & 81.3 & 94.7 \\
         Mformer-HR \cite{patrick2021keeping} & T & 311 & $959\times 3\times 10$ & $16\times 336^2$ & 81.1 & 95.2 \\
         VideoMamba-M \cite{li2024videomambastatespacemodel} & M & 74& $202\times 3\times 4$& $16\times 224^2$& 81.9 & 95.4   \\ 
         \hline
         \hline
         VideoMamba-S \cite{li2024videomambastatespacemodel} & M & 26& $34\times 3\times 4$& $8\times 224^2$& 79.3 & 94.2   \\ 
          VideoMamba-S \cite{li2024videomambastatespacemodel} & M & 26& $68\times 3\times 4$& $16\times 224^2$& 80.8 & 94.8\\
         VideoMamba-S \cite{li2024videomambastatespacemodel} & M & 26& $135\times 3\times 4$& $32\times 224^2$& 81.5 & 95.2\\
         \hline
         VideoSEMA (Ours) &  Ml & 31 & $46 \times 3 \times 4$ & $8\times 224^2$& \textbf{81.3} & \textbf{94.9}\\
          VideoSEMA (Ours)&  Ml & 31 & $87\times 3 \times 4$ & $16\times 224^2$& \textbf{82.4} & \textbf{95.4}\\
         VideoSEMA (Ours) &  Ml & 31 & $172\times 3\times 4$ & $32\times 224^2$& \textbf{82.6} & \textbf{95.2} \\
         \hline
    \end{tabular}

    \label{tab:K400_SOTA}
\end{table}

\begin{table}[ht!]
    \caption{Performance reported on  standard SSv2 dataset. Here $F\times m\times n$ denotes by $F$ the total FLOPs, $m$ the number of testing segments, and $n$ the number of testing crops. Res: resolution, T: transformer.  P(M): parameter size in millions. T1: top-1, T5: top-5. C: CNN, CT: CNN +T, M: Mamba, Ml: Mamba-like.}
    \centering
    \begin{tabular}{|l|c|c c| c |c c| }
    \hline
         Method & Type  &  P(M) &  FLOPs (G) & Res& T1(\%) & T5(\%)\\
         \hline
         CT-Net$_{R50}$ \cite{li2021ctnet} & C & 21& $75\times 1\times 1$& $16\times 224^2$& 64.5 & 89.3\\
         TDN$_{R50}$ \cite{Wang_2021_CVPR} & C & 26& $75\times 1\times 1$& $16\times 224^2$& 65.3 & \textbf{91.6}\\
         WLiT \cite{sun2023wlit} & T & 22 & $50\times 3 \times 1$ & $16\times 224^2$ & 66.3 & 91.5 \\
         VideoMAE \cite{NEURIPS2022_VideoMAE} & T & 22& $57\times 2\times 3$& $16\times 224^2$& 66.8 & 90.3\\
         MViTv1-B \cite{Fan_2021_ICCV} & CT & 37& $71\times 3\times 1$& $16\times 224^2$& 64.7 & 89.2\\
         \hline
         \hline
         VideoMamba-S \cite{li2024videomambastatespacemodel} & M & 26& $34\times 3\times 2$& $8\times 224^2$& 65.2 & 89.6\\ 
         VideoMamba-S \cite{li2024videomambastatespacemodel} & M & 26& $68\times 3 \times 2$& $16\times 224^2$& 66.0 & 90.2\\
         \hline
         VideoSEMA (Ours) & Ml & 31& $46\times 3 \times 2$ & $8\times 224^2$& \textbf{67.2} & 91.0 \\ 
         VideoSEMA (Ours) & Ml & 31& $87\times 3\times 2$ & $16\times 224^2$& \textbf{67.3} & 90.4 \\
         \hline
    \end{tabular}

    \label{tab:SthSthv2}
\end{table}

\subsection{Dataset}
We evaluate our approach on two widely used large-scale video action recoginition benchmarks: Kinetic-400 (K400) \cite{kay2017kineticshumanactionvideo} and Something-Something v2 (SSv2) \cite{8237884}. K400 contains around 240000 training, 20000 validation, and 40000 testing videos spanning 400 human action categories, with clips source from YouTube and trimmed to focus on a single action that lasted around 10 seconds. The dataset is focused on human action from diverse scenes, actors, and viewpoints, making it a standard benchmark for learning high-level semantic. In contrast, SSv2 consists of around 220000 videos, with approximately 170000 training, 25000 validation and 27000 testing videos that last around 2-6 seconds. SSv2 places a strong temporal reasoning and fine-grained actions. Together, these datasets provide complementary evaluation settings.
\subsection{Setting}
A common strategy to train a video model is to pretrain a model with an image-based architecture \cite{gberta_2021_ICML_space_time_attention,li2024videomambastatespacemodel, Liu_2022_CVPR, NEURIPS2022_VideoMAE} on ImageNet-1K or ImageNet-21K, and then inflate the image model into a video model. For fair comparison with VideoMamba, we pretrain the SEMA model on ImageNet-1K, then incorporate the temporal component and train the model on K400 and SSv2 to obtain the VideoSEMA results. For SEMA, we use the same training setup as described in \cite{tran2025semascalableefficientmamba}. We use the setting of the T model variant, i.e.  4 stages with hidden dimension of 64, 128, 256, and 512 respectively. To train VideoSEMA, we use a similar setup to VideoMamba. In particular, for K400 we use a set of 5 warmup epochs, 50 total epochs, a 0.35 stochastic depth rate, and 0.05 weight decay, with an initial learning rate of $2\times 10^{-4}$ using the AdamW optimizer. For SSv2, we use a set of 5 warmup epochs, 30 total epochs, a 0.35 stochastic depth rate, and 0.05 weight decay, with initial learning rate of $4\times 10^{-4}$. We train and test VideoSEMA and VideoMamba on 8 NVIDIA RTX A6000 GPUs, each with 46G of memory. Code will be available upon publication.

\subsection{Experimental Results}

\begin{figure}
    \centering
    \begin{subfigure}{\textwidth}
        \includegraphics[width=\textwidth]{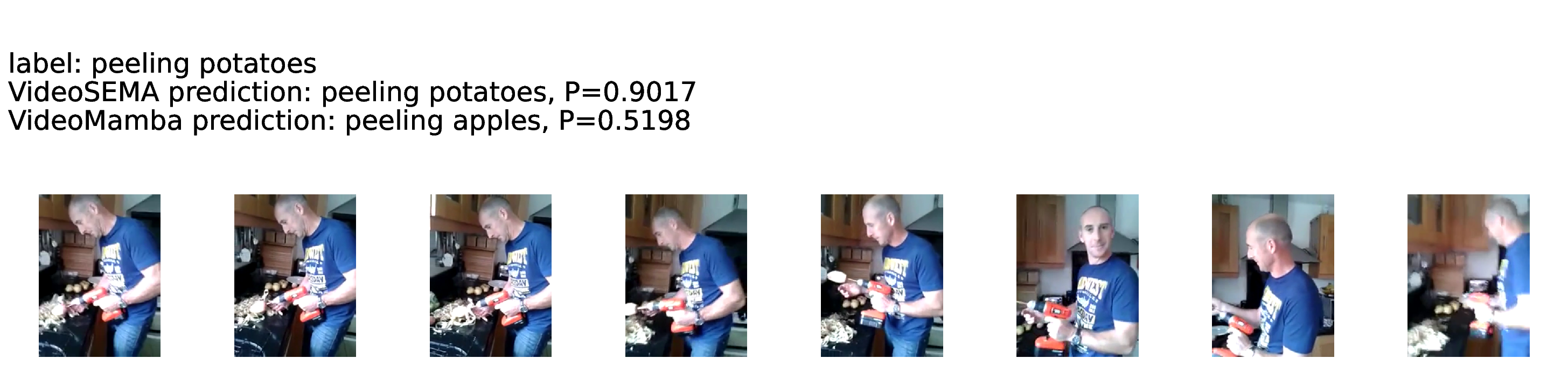}
        \caption{example depicts a person peeling potatoes, where VideoSEMA correctly predicts the action with high ($P=0.9$), while VideoMamba incorrectly classifies it as peeling apples.}
    \end{subfigure}
        \begin{subfigure}{\textwidth}
        \includegraphics[width=\textwidth]{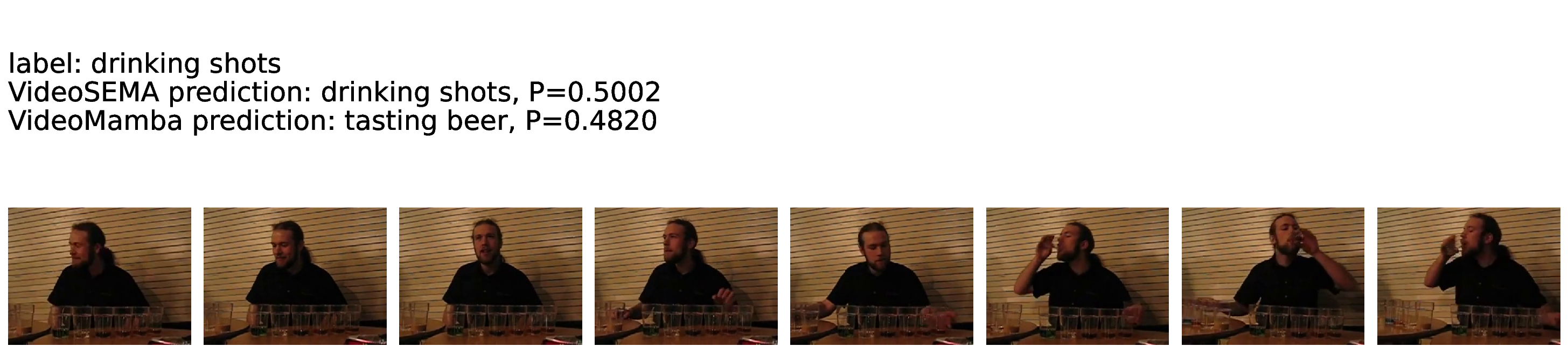}
        \caption{example shows a person drinking shots, which VideoSEMA correctly recognizes with moderate confidence $P=0.5$.}
    \end{subfigure}
        \begin{subfigure}{\textwidth}
        \includegraphics[width=\textwidth]{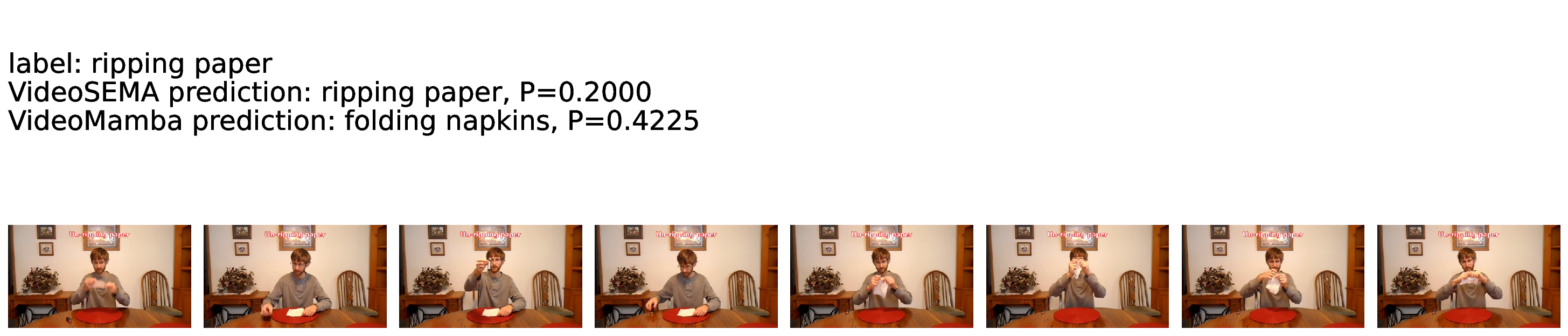}
        \caption{example contains a video of a person ripping paper played in reverse time; VideoSEMA correctly identifies the action with low confidence, whereas VideoMamba fails to classify it correctly despite assigning moderate confidence to its prediction.}
    \end{subfigure}
    \caption{Qualitative comparison of VideoSEMA and VideoMamba on the K400 test set. Shown are representative test samples along with the predicted labels and corresponding highest confidence scores.}
    \label{fig:k400-visual}
\end{figure}


\subsubsection{K400}
We present the overall performance of VideoSEMA on K400 dataset in Table \ref{tab:K400_SOTA}. Compared to VideoMamba under a similar computational budget, VideoSEMA's top-1 accuracies are 1-2\% better across different input resolutions. Compared to other state of the art methods under similar training methodologies and computational budgets shown in the first row group of Table \ref{tab:K400_SOTA}, VideoSEMA performs significantly better in both top-1 and top-5 metrics. For example, at an input resolution of $16\times 224^2$, it is $1.4\%$ better than MViTv2-S in top-1 and $0.8\%$ in top-5.

Notably, the second row group of Table \ref{tab:K400_SOTA} includes models with substantially larger parameter counts and FLOPs. Despite operating under much smaller computational budget, VideoSEMA consistently achieves superior performance. This demonstrates that VideoSEMA is a 

In addition, we present a few visual example comparisons between VideoSEMA and VideoMamba on somewhat challenging predictions in Fig. \ref{fig:k400-visual}. Example (A) depicts a video of a person peeling potatoes. VideoSEMA correctly predicts the action with a high probability of $90\%$, while VideoMamba incorrectly classifies it as the similar label peeling apples. The global action of peeling is correct for both models; however, this particular video is challenging due to the local distinction between apples and potatoes, which occupy only a small region of the image. This could be explained by the window attention of SEMA, which keeps the fine details. Example (B) shows a person drinking shots. VideoSEMA correctly identifies the action with a medium probability of $50\%$, while VideoMamba misidentifies it as tasting beer. Drinking shots involves rapidly consuming a small amount of liquid, while tasting beer is a slower action. This shows that the attention in time between frames allows VideoSEMA to understand quick changes in motion between frames, while VideoMamba processes these tokens in sequential order, thus reducing its temporal capability. Lastly, example (C) shows a video of a person ripping paper recorded in reverse time. VideoSEMA classifies it correctly with low confidence of $20\%$, while VideoMamba predicts folding napkins. Because the video plays in reverse, the model needs strong temporal understanding. VideoSEMA with temporal attention allows the model to access frames in both directions simultaneously, which helps the model prediction. In contrast, VideoMamba’s forward-direction processing observes the person putting the paper together and thus predicts folding napkins.

\begin{table}[ht!]
    \caption{Performance (ablation study) of VideoSEMA using various temporal processing units on K400 dataset. Here $F\times m\times n$ denotes by $F$ the total FLOPs, $m$ the number of testing segments, $n$ the number of testing crops.}
    \centering
    \begin{tabular}{|l|c|c c| c |c|}
    \hline
         Time Component  & \# Params (M) &  FLOPs (G) & Res& Top-1 (\%) \\
         \hline
         3DConv (ker=7) & 26 & $40 \times 3 \times 4$ & $8\times 224^2$ &80.0\\
         3DConv (ker=7) & 26 & $76 \times 3 \times 4$ & $16\times 224^2$ &79.3\\
         3DConv (ker=13) & 26 & $76 \times 3 \times 4$ & $16\times 224^2$ &80.8\\
         Mamba & 36 & $47\times 3\times 4$ & $8\times 224^2$& 76.3\\
         Attention & 31 & $46\times 3\times 4$ & $8\times 224^2$ & \textbf{81.3}\\
         
        \hline
    \end{tabular}

    \label{tab:K400_model_choice}
\end{table}

\subsubsection{SSv2}
We present the overall performance of VideoSEMA on the SSv2 dataset in Table \ref{tab:SthSthv2}. We observe that with 8 frames inputs, VideoSEMA achieve a $2\%$ improvement over VideoMamba-S, and with 16 frames inputs, it outperforms VideoMAE by $0.5\%$. Notably, using only 8 frames, VideoSEMA is able to outperform other models that utilize all 16 frames. This demonstrates that VideoSEMA achieves strong efficacy and temporal efficiency. We also note that the performance improvement with 16 frames is smaller compared to that with 8 frames. One possible explanation is the short temporal duration of videos in SSv2, thus limiting the benefit of longer input sequences.

\begin{table}[]
    \centering
    \caption{Top-1 accuracy (\%) at higher input resolutions to models pretrained on  $16\times 224^2$ videos of K400, without fine-tuning.}
    \begin{tabular}{|c|c|c|c|}
    \hline
         Method & $16\times 224^2$ (baseline) & $16\times 512^2$ & $16\times 1024^2$ \\
         \hline
         VideoMamba& 80.8 & 74.4 & 40.8 \\
         VideoSEMA& 82.4 & 75.3 & 54.6\\
         \hline
    \end{tabular}
    
    \label{tab:k400-scale-up}
\end{table}

\subsection{Ablation Study}
We extend the SEMA model from image domain to processing videos by adding a temporal processing unit. There are many standard choices for this component such as convolution, attention, and Mamba. In this subsection, we examine the performance of each choice. For convolution, to maintain relatively low parameter count, we employed a depthwise convolution. For Mamba, we use a single directional Mamba from VideoMamba \cite{li2024videomambastatespacemodel}. And lastly, for attention we use softmax full attention. From Table \ref{tab:K400_model_choice}, we observe that on K400, convolution performs relatively well, and as the number of frames increases, we need larger temporal convolutional kernel. Mamba on the other hand, performs poorly as a temporal processing unit. One possible explanation is that Mamba is applied only in the temporal dimension, resulting in independent operations on spatial positions. Lastly, since attention provides the best trade-off between accuracy and efficiency, we use attention as a temporal processing unit for VideoSEMA on the datasets here.

\subsection{Scalability to Higher Resolution Videos}
SEMA is designed to handle high-resolution frames. In this subsection, we examine the adaptability of the model when applied to larger input resolutions. We use the model pretrained on $224^2$ and evaluate it on larger input sizes of $512^2$ and $1024^2$ with 16 frames on K400. From Tab. \ref{tab:k400-scale-up}, we observe that VideoMamba and VideoSEMA perform similarly on $224^2$; however, at $1024^2$, the performance of VideoMamba degrades much more quickly compared to VideoSEMA. In particular, the performance gap between VideoMamba and VideoSEMA is about 13.8 percentage points. This demonstrates the robustness and scalability of VideoSEMA for large input resolutions.


         



\section{Conclusion}
We introduced a space-time 
attention model (VideoSEMA) consisting of a scalable and efficient Mamba-like attention (SEMA) in space and a regular attention in time. For moderate number of video frames, the model out-performs recent space-time transformers and video-mamba of larger (comparable) sizes on benchmark K400 (SSv2) datasets. 
In future work, we plan to scale the model to longer videos by using dilated/sparse attention \cite{Longnet_2023} in time and evaluate the model's robustness on higher spatial resolutions across more datasets. 

\section*{Acknowledgments}
The work was partly supported by NSF grants DMS-2219904, DMS-2309520, and a Qualcomm Gift Award. NTT was also funded by a Faculty Endowed Fellowship and the Graduate Scholar Success Fund from the University of California, Irvine.
\bibliography{bib}
\bibliographystyle{plain}
\end{document}